\documentclass[english,a4paper,12pt]{article}

\usepackage{amsxtra, amsfonts, amstext, amsmath, mathtools, dsfont}
\usepackage{amsthm}
\usepackage{booktabs}
\usepackage{nicefrac}
\usepackage{xspace}
\usepackage{url}
\usepackage{graphics,color}
\usepackage[algo2e,ruled,vlined,linesnumbered]{algorithm2e}
\usepackage{wrapfig}

\allowdisplaybreaks[1]
\newtheorem{theorem}{Theorem}
\newtheorem{lemma}[theorem]{Lemma}

\newtheorem{definition}[theorem]{Definition}

\newcommand{\oea}{\mbox{$(1 + 1)$~EA}\xspace}

\newcommand{\OM}{\textsc{OM}\xspace}
\newcommand{\onemax}{\textsc{OneMax}\xspace}
\newcommand{\LO}{\textsc{Leading\-Ones}\xspace}
\newcommand{\leadingones}{\LO}

\DeclareMathOperator{\jump}{\textsc{Jump}}
\DeclareMathOperator{\rand}{rand}

\newcommand{\R}{\ensuremath{\mathbb{R}}}

\newcommand{\N}{\ensuremath{\mathbb{N}}} 

\DeclareMathOperator{\Bin}{Bin}

\newcommand{\eps}{\varepsilon} 

\newcommand{\merk}[1]{\textbf{\textcolor{red}{#1}}}
\newcommand{\ignore}[1]{}

\newcommand{\assign}{\leftarrow}

\begin{document}

\title{\vspace*{-1cm}Lower Bounds from Fitness Levels Made Easy\thanks{This work was supported by a public grant as part of the
Investissements d'avenir project, reference ANR-11-LABX-0056-LMH,
LabEx LMH, and by the Deutsche Forschungsgemeinschaft (DFG), grant FR 2988/17-1.
}
\thanks{Extended version of the conference paper~\cite{DoerrK21gecco} accepted for publication in the proceedings of \emph{GECCO 2021}.}}


\author{Benjamin Doerr\setcounter{footnote}{6}\thanks{Laboratoire d'Informatique (LIX), CNRS, \'Ecole Polytechnique, Institut Polytechnique de Paris, Palaiseau, France} \and Timo K{\"o}tzing\thanks{Hasso Plattner Institute, Potsdam, Germany}}

\maketitle

\begin{abstract}
One of the first and easy to use techniques for proving run time bounds for evolutionary algorithms is the so-called method of fitness levels by Wegener. It uses a partition of the search space into a sequence of levels which are traversed by the algorithm in increasing order, possibly skipping levels. An easy, but often strong upper bound for the run time can then be derived by adding the reciprocals of the probabilities to leave the levels (or upper bounds for these). Unfortunately, a similarly effective method for proving lower bounds has not yet been established. The strongest such method, proposed by Sudholt (2013), requires a careful choice of the viscosity parameters $\gamma_{i,j}$, $0 \le i < j \le n$. 

In this paper we present two new variants of the method, one for upper and one for lower bounds. Besides the level leaving probabilities, they only rely on the probabilities that levels are visited at all. We show that these can be computed or estimated without greater difficulties and apply our method to reprove the following known results in an easy and natural way. (i) The precise run time of the (1+1) EA on \textsc{LeadingOnes}. (ii) A lower bound for the run time of the (1+1) EA  on \textsc{OneMax}, tight apart from an $O(n)$ term. (iii) A lower bound for the run time of the (1+1) EA on long $k$-paths. We also prove a tighter lower bound for the run time of the (1+1) EA on jump functions by showing that, regardless of the jump size, only with probability $O(2^{-n})$ the algorithm can avoid to jump over the valley of low fitness.
%
%
%
\end{abstract}

\sloppy{
\section{Introduction}

The theory of evolutionary computation aims at explaining the behavior of evolutionary algorithms, for example by giving detailed run time analyses of such algorithms on certain test functions, defined on some search space (for this paper we will focus on $\{0,1\}^n$). The first general method for conducting such analyzes is the \emph{fitness level method (FLM)} \cite{Wegener01,Wegener02}. The idea of this method is as follows. 
 We partition the search space into a number $m$ of sections (``levels'') in a linear fashion, so that all elements of later levels have better fitness than all elements of earlier levels. For the algorithm to be analyzed we regard the best-so-far individual and the level it is in. Since the best-so-far individual can never move to lower levels, it will visit each level at most once (possibly staying there for some time). Suppose we can show that, for any level $i < m$ which the algorithm is currently in, the probability to leave this level is at least $p_i$. Then, bounding the expected waiting for leaving a level $i$ by $1/p_i$, we can derive an upper bound for the run time of $\sum_{i=1}^{m-1} 1/p_i$ by pessimistically assuming that we visit (and thus have to leave) each level $i < m$ before reaching the target level $m$. The fitness level method allows for simple and intuitive proofs and has therefore frequently been applied. Variations of it come with tail bounds~\cite{Witt14}, work for parallel EAs~\cite{LassigS14}, or admit non-elitist EAs~\cite{Lehre11,DangL16algo,CorusDEL18,DoerrK19}.

While very effective for proving upper bounds, it seems much harder to use fitness level arguments to prove lower bounds (see Theorem~\ref{thm:FLMlow} for an early attempt). The first (and so far only) to devise a fitness level-based lower bound method that gives competitive bounds was Sudholt \cite{Sudholt13}. His approach uses viscosity parameters $\gamma_{i,j}$, $0 \le i < j \le n$, which control the probability of the algorithm to jump from one level $i$ to a higher level $j$ (see Section~\ref{sec:flmWithViscosities} for details). While this allows for deriving strong results, the application is rather technical due to the many parameters and the restrictions they have to fulfill.

In this paper, we propose a new variant of the FLM for lower bounds, which is easier to use and which appears more intuitive. For each level $i$, we regard the \emph{visit probability} $v_i$, that is, the probability that level $i$ is visited at all during a run of the algorithm. This way we can directly characterize the run time of the algorithm as $\sum_{i=1}^{m-1} v_i/p_i$ when $p_i$ is the precise probability to leave level $i$ independent of where on level $i$ the algorithm is. When only estimates for these quantities are known, e.g., because the level leaving probability is not independent from the current state, then we obtain the corresponding upper or lower bounds on the expected run time (see Section~\ref{sec:flmWithVisitProbabilities} for details).

We first use this method to give the precise expected run time of the \oea on \leadingones in Section~\ref{sec:leadingones}. While this run time was already well-understood before, it serves as a simple demonstration of the ease with which our method can be applied.

Next, in Section~\ref{sec:onemax}, we give a bound on the expected run time of the \oea on \onemax, precise apart from terms of order $\Theta(n)$. Such bounds have also been known before, but needed much deeper methods (see Section~\ref{ssec:onemaxlit} for a detailed discussion). Sudholt's lower bound method has also been applied to this problem, but gave a slightly weaker bound deviating from the truth by an $O(n \log\log n)$ term. In addition to the precise result, we feel that our FLM with visit probabilities gives a clearer structure of the proof than the previous works.

In Section~\ref{sec:jump}, we prove tighter lower bounds for the run time of the \oea on jump functions. We do so by determining (asymptotically precise) the probability that in a run of the \oea on a jump function the algorithm does not reach a non-optimal search point outside the fitness valley (and thus does not have to jump over this valley). Interestingly, this probability is only $O(2^{-n})$ regardless of the jump size (width of the valley).

Finally, in Section~\ref{sec:longKPaths}, we consider the \oea on so-called long $k$-paths. We show how the FLM with visit probabilities can give results comparable to those of the FLM with viscosities while again being much simpler to apply.

\section{The {\oea}}
\label{sec:algo}

In this paper we consider exactly one randomized search heuristic, the \oea. It maintains a single individual, the best it has seen so far. Each iteration it uses standard bit mutation with mutation rate $p \in (0,1)$ (flipping each bit of the bit string independently with probability $p$) and keeps the result if and only if it is at least as good as the current individual under a given fitness function $f$. We give a more formal definition in Algorithm~\ref{alg:oea}.
\begin{algorithm2e}
	Let $x$ be a uniformly random bit string from $\{0,1\}^n$\;
	\While{optimum not reached}{%
		$y \assign \mbox{mutate}_p(x)$\;
		\lIf{$f(y) \geq f(x)$}{$x \assign y$}
	}
	\caption{The \oea to maximize $f : \{0,1\}^n \to \R$.}
	\label{alg:oea}
\end{algorithm2e}

\section{The Fitness Level Methods}

The fitness level method is typically phrased in terms of a \emph{fitness-based partition}, that is, a partition of the search space into sets $A_1,\ldots,A_m$ such that elements of later sets have higher fitness. We first introduce this concept and abstract away from it to ease the notation. After this, in Section~\ref{sec:originalFLM}, we state the original FLM. In Section~\ref{sec:flmWithViscosities} we describe the lower bound based on the FLM from Sudholt \cite{Sudholt13}, before presenting our own variant, the FLM with visit probabilities, in Section~\ref{sec:flmWithVisitProbabilities}.

\subsection{Level Processes}

\begin{definition}[Fitness-Based Partition {\cite{Wegener02}}]
Let $f: \{0,1\}^n \rightarrow \R$ be a fitness function. A partition $A_1,\ldots,A_m$ of $\{0,1\}^n$ is called a \emph{fitness-based partition} if for all $i,j \leq m$ with $i < j$ and $x \in A_i$, $y \in A_j$, we have $f(x) < f(y)$.
\end{definition}

We will use the shorthands $A_{\geq i} = \bigcup_{j=i}^m A_j$ and $A_{\leq i} = \bigcup_{j=1}^i A_j$.

In order to simplify our notation, we focus on processes on $[1..m]$ (the levels) with underlying Markov chain as follows.
\begin{definition}[Non-decreasing Level Process]\label{def:levelProcess}
A stochastic process $(X_t)_t$ on $[1..m]$ is called a \emph{non-decreasing level process} if and only if (i)~there exists a Markov process $(Y_t)_t$ over a state space $S$ such that there is an $\ell: S \rightarrow [1..m]$ with $\ell(Y_t) = X_t$ for all $t$, and (ii)~the process $(X_t)_t$ is non-decreasing, that is, we have $X_{t+1} \ge X_t$ with probability one for all $t$. 
\end{definition}
We later want to analyze algorithms in terms of non-decreasing level processes, making the transition as follows. Suppose we have an algorithm with state space $\{0,1\}^n$. Denoting by $Y_t$
the best among the first $t$ search points generated by the algorithm, this defines a Markov Chain $(Y_t)_t$ in the state space $S = \{0,1\}^n$, the run of the algorithm. Further, suppose the algorithm optimizes a fitness function $f$ such that the state of the algorithm is non-decreasing in terms of fitness. In order to get a non-decreasing level process, we can now define any fitness-based partition and get a corresponding \emph{level function} $\ell: S \rightarrow [1..m]$ by mapping any $x \in S$ to the unique $i$ with $x \in A_i$. Then the process $(\ell(Y_t))_t$ is a non-decreasing level process.

The main reason for us to use the formal notion of a level process is the property formalized in the following lemma. Essentially, if a level process makes progress with probability at least $p$ in each iteration (regardless of the precise current state), then the expected number of iterations until the process progresses is at most $1/p$. This situation resembles a geometric distribution, but does not assume independence of the different iterations (one could show that the time to progress is stochastically dominated by a geometric distribution with success rate $p$, but we do not need this level of detail).

\begin{lemma}\label{lem:geometricDistribution}
Let $(X_t)_t$ be a non-decreasing level process with underlying Markov chain $(Y_t)_t$ and level function $\ell$. Assume $X_t$ starts on some particular level. Let $p \in (0,1]$ be a lower bound on the probability for level process to leave this level regardless of the state of the underlying Markov chain. Then the expected first time $t$ such that $X_t$ changes is at most $1/p$.

Analogously, if $p$ is an upper bound, the expected time $t$ such that $X_t$ changes is at least $1/p$.
\end{lemma}

\begin{proof}
We let $(Z_t)_t$ be the stochastic process on $\{0,1\}$ such that $Z_t$ is $1$ if and only if $X_t > X_0$. According to our assumptions, we have, for all $t$ before the first time that $Z_t=1$, that $E[Z_{t+1} - Z_t \mid Z_t] \geq p$. From the additive drift theorem \cite{HeY01,Lengler20bookchapter} we obtain that the expected first time such that $Z_t = 1$ is bounded by $1/p$ as desired. The ``analogously'' clause follows analogously.
\end{proof}

\subsection{Original Fitness Level Method}
\label{sec:originalFLM}

The following theorem contains the original Fitness Level Method and makes the basic principle formal.

\begin{theorem}[Fitness Level Method, upper bound {\cite{Wegener02}}]\label{thm:classic}
Let $(X_t)_t$ be a non-decreasing level process (as detailed in Definition~\ref{def:levelProcess}).

 For all $i \in [1..m-1]$, let $p_i$ be a lower bound on the probability of a state change of $(X_t)_t$, conditional on being in state~$i$. Then the expected time for $(X_t)_t$ to reach the state $m$ is 
\[
E[T] \leq \sum_{i=1}^{m-1} \frac{1}{p_i}.
\]
\end{theorem}

This bound is very simple, yet strong. It is based on the idea that, in the worst case, all levels have to be visited sequentially. Note that one can improve this bound (slightly) by considering only those levels which come after the (random) start level $X_0$ (by changing the start of the sum to $X_0$ instead of $1$). Intuitively, low levels that are never visited do not need to be left.

There is a lower bound based on the observation that at least the initial level has to be left (if it was not the last level).

\begin{theorem}[Fitness Level Method, lower bound {\cite{Wegener02}}]\label{thm:FLMlow}
Let $(X_t)_t$ be a non-decreasing level process (as detailed in Definition~\ref{def:levelProcess}).

For all $i \in [1..m-1]$, let $p_i$ be an upper bound on the probability of a state change, conditional on being in state $i$. Then the expected time for $(X_t)_t$ to reach the state $m$ is 
$$
E[T] \geq \sum_{i=0}^{m-1} \Pr[X_0 = i] \frac{1}{p_i}.
$$	
\end{theorem}

This bound is very weak since it assumes that the first improvement on the initial search point already finds the optimum.

We note, very brief{}ly, that a second main analysis method, \emph{drift analysis}, also has additional difficulties with lower bounds. Additive drift~\cite{HeY01}, multiplicative drift~\cite{DoerrJW12algo}, and variable drift~\cite{MitavskiyRC09,Johannsen10} all easily give upper bounds for run times, however, only the additive drift theorem yields lower bounds with the same ease. The existing multiplicative~\cite{Witt13,DoerrDK18,DoerrKLL20} and variable~\cite{DoerrFW11,FeldmannK13,GiessenW18,DoerrDY20} drift theorems for lower bounds all need significantly stronger assumptions than their counterparts for upper bounds.

\subsection{Fitness Level Method with Viscosity}
\label{sec:flmWithViscosities}

While the upper bound above is strong and useful, the lower bound is typically not strong enough to give more than a trivial bound. Sudholt~\cite{Sudholt13} gave a refinement of the method by considering bounds on the transition probabilities from one level to another.

\begin{theorem}[Fitness Level Method with Viscosity, lower bound {\cite{Sudholt13}}]
Let $(X_t)_t$ be a non-decreasing level process (as detailed in Definition~\ref{def:levelProcess}).
Let $\chi,\gamma_{i,j} \in [0,1]$ and $p_i \in (0,1]$ be such that
\begin{itemize}
	\item for all $t$, if $X_t = i$, the probability that $X_{t+1} = j$ is at most $p_i \cdot \gamma_{i,j}$;
	\item $\sum_{j=i+1}^m \gamma_{i,j} = 1$; and
	\item for all $j>i$, we have $\gamma_{i,j} \geq \chi \sum_{k=j}^m \gamma_{i,k}$.
\end{itemize}
Then the expected time for $(X_t)_t$ to reach the state $m$ is 
$$
E[T] \geq \sum_{i=1}^{m-1} \Pr[X_0 = i] \; \chi \sum_{j=i}^{m-1} \frac{1}{p_j}.
$$	
\end{theorem}
This result is much stronger than the original lower bound from Fitness Level Method, since now the leaving probabilities of all segments are part of the bound, at least with a fractional impact prescribed by $\chi$. The weakness of the method is that $\chi$ has to be defined globally, the same for all segments~$i$.

There is also a corresponding upper bound as follows.
\begin{theorem}[Fitness Level Method with Viscosity, upper bound {\cite{Sudholt13}}]
Let $(X_t)_t$ be a non-decreasing level process (as detailed in Definition~\ref{def:levelProcess}).
Let $\chi,\gamma_{i,j} \in [0,1]$ and $p_i \in (0,1]$ be such that
\begin{itemize}
	\item for all $t$, if $X_t = i$, the probability that $X_{t+1} = j$ is at least $p_i \cdot \gamma_{i,j}$;
	\item $\sum_{j=i+1}^m \gamma_{i,j} = 1$;
	\item for all $j>i$, we have $\gamma_{i,j} \leq \chi \sum_{k=j}^m \gamma_{i,k}$; and
	\item for all $j \leq m-2$, we have $(1-\chi)p_j \leq p_{j+1}$.
\end{itemize}
Then the expected time for $(X_t)_t$ to reach the state $m$ is 
$$
E[T] \leq \sum_{i=1}^{m-1} \Pr[X_0 = i] \; \left(\frac{1}{p_j}+\chi \sum_{j=i+1}^{m-1} \frac{1}{p_j}\right).
$$	
\end{theorem}

\subsection{Fitness Level Method with Visit Probabilities}
\label{sec:flmWithVisitProbabilities}

In this paper, we give a new FLM theorem for proving lower bounds. The idea is that exactly all those levels that have ever been visited need to be left; thus, we can use the expected waiting time for leaving a specific level multiplied with the probability of visiting that level at all. The following theorem makes this idea precise for lower bounds; Theorem~\ref{thm:fbp_visit_upper} gives the corresponding upper bound. We note that for the particular case of the optimization of the \leadingones problem via $(1+1)$-type elitist algorithms, our bounds are special cases of~\cite[Lemma~5]{DoerrJWZ13} and~\cite[Theorem~3]{Doerr19tcs}.

\begin{theorem}[Fitness Level Method with visit probabilities, lower bound]\label{thm:fbp_visit_lower}
Let $(X_t)_t$ be a non-decreasing level process (as detailed in Definition~\ref{def:levelProcess}).
For all $i \in [1..m-1]$, let $p_i$ be an upper bound on the probability of a state change of $(X_t)_t$, conditional on being in state~$i$. 
Furthermore, let $v_i$ be a lower bound on the probability of there being a $t$ such that $X_t = i$.
Then the expected time for $(X_t)_t$ to reach the state $m$ is 
$$
E[T] \geq \sum_{i=1}^{m-1} \frac{v_i}{p_i}.
$$	
\end{theorem}

\begin{proof}
For each $i < m$, let $T_i$ be the (random) time spent in level~$i$. Thus,
$$
T = \sum_{i=1}^{m-1} T_i.
$$
Let now $i < m$. We want to show that $E[T_i] \geq v_i/p_i$. We let $E$ be the event that the process ever visits level $i$ and compute
$$
E[T_i] = E[T_i \mid E]\Pr[E] + E[T_i \mid \overline{E}]\Pr[\overline{E}] \geq E[T_i \mid E]v_i.
$$
For all $t$ with $X_t = i$, with probability at most $p_i$, we have $X_{t+1} > i$. Thus, using Lemma~\ref{lem:geometricDistribution}, the expected time until a search point with $X_k > i$ is found is at least $1/p_i$, giving $E[T_i \mid E] \geq 1/p_i$ as desired.
\end{proof}

A strength of this formulation is that skipping levels due to a higher initialization does not need to be taken into account separately (as in the two previous lower bounds), it is part of the visit probabilities. A corresponding upper bound follows with analogous arguments.

\begin{theorem}[Fitness Level Method with visit probabilities, upper bound]\label{thm:fbp_visit_upper}
Let $(X_t)_t$ be a non-decreasing level process (as detailed in Definition~\ref{def:levelProcess}).

For all $i \in [1..m-1]$, let $p_i$ be a lower bound on the probability of a state change of $(X_t)_t$, conditional on being in state $i$. 
Furthermore, let $v_i$ be an upper bound on the probability there being a $t$ such that $X_t = i$.
Then the expected time for $(X_t)_t$ to reach the state $m$ is 
$$
E[T] \leq \sum_{i=1}^{m-1} \frac{v_i}{p_i}.
$$	
\end{theorem}


In a typical application of the method of the FLM, finding good estimates for the leaving probabilities is easy. It is more complicated to estimate the visit probabilities accurately, so we propose one possible approach in the following lemma.

\begin{lemma}\label{lem:visitprob}
Let $(Y_t)_t$ be a Markov-process over state space $S$ and $\ell: S \rightarrow [1..m]$ a level function. For all $t$, let $X_t = \ell(Y_t)$ and suppose that $(X_t)_t$ is non-decreasing. Further, suppose that $(X_t)_t$ reaches state $m$ after a finite time with probability $1$.


Let $i < m$ be given. For any $x \in S$ and any set $M \subseteq S$, let $x \rightarrow M$ denote the event that the Markov chain with current state $x$ transitions to a state in $M$. For all $j$ let $A_j = \{s \in S \mid \ell(s) = j\}$. Suppose there is $v_i$ such that, for all $x \in A_{\leq i-1}$ with $\Pr[x \rightarrow A_{\geq i}] > 0$,
$$
\Pr[x \rightarrow A_i \mid x \rightarrow A_{\geq i}] \geq v_i,
$$
and 
$$
\Pr[Y_0 \in A_i \mid Y_0 \in A_{\geq i}] \geq v_i.
$$
Then $v_i$ is a lower bound for visiting level $i$ as required by Theorem~\ref{thm:fbp_visit_lower}.
\end{lemma}

\begin{proof}
Let $T$ be minimal such that $Y_T \in A_{\geq i}$. Then the probability that level $i$ is being visited is $\Pr[Y_T \in A_i]$, since $(X_t)_t$ is non-decreasing.

By the law of total probability we can show the claim by showing it first conditional on $T=0$ and then conditional on $T\neq 0$.

We have that $T = 0$ is equivalent to $Y_0 \in A_{\geq i}$, thus we have $\Pr[Y_T \in A_i \mid T = 0] \geq v_i$ from the second condition in the statement of the lemma.

Otherwise, let $x = Y_{T-1}$. Since $Y_T \in A_{\geq i}$,
\begin{eqnarray*}
\Pr[Y_T \in A_i \mid T \neq 0]
 & = & \Pr[Y_T \in A_i \mid Y_T \in A_{\geq i}, T \neq 0]\\
 & = & \Pr[x \rightarrow A_i \mid x \rightarrow A_{\geq i}, T \neq 0]\\
 & = & \Pr[x \rightarrow A_i \mid x \rightarrow A_{\geq i}].
\end{eqnarray*}
As $T$ was chosen minimally, we have $x \not\in A_{\geq i}$ and thus get the desired bound from the first condition in the statement of the lemma.
\end{proof}
Implicitly, the lemma suggests to take the minimum of all these conditional probabilities over the different choices for $x$. Note that this estimate might be somewhat imprecise since worst-case $x$ might not be encountered frequently. Also note that a corresponding upper bound for Theorem~\ref{thm:fbp_visit_upper} follows analogously.

\section{The Precise Run Time for \leadingones}
\label{sec:leadingones}

One of the classic fitness functions used for analyzing the optimization behavior of randomized search heuristics is the \leadingones function. Given a bit string $x$ of length $n$, the  \leadingones value of $x$ is defined as the number of $1$s in the bit string before the first $0$ (if any). 
In parallel independent work, the precise expected run time of the \oea on the \leadingones benchmark function was determined in \cite{BottcherDN10,Sudholt13}. 
Even more, the distribution of the run time was determined with variants of the FLM in \cite{DoerrJWZ13,Doerr19tcs}.
As a first simple application of our methods, we now determine the precise run time of the \oea on \leadingones via Theorems~\ref{thm:fbp_visit_lower} and~\ref{thm:fbp_visit_upper}.

\begin{theorem}\label{thm:LO}
Consider the \oea optimizing \leadingones with mutation rate $p$. Let $T$ be the (random) time for the \oea to find the optimum. Then
$$
E[T] = \frac{1}{2} \sum_{i=0}^{n-1} \frac{1}{(1-p)^i p}.
$$	
\end{theorem}
\begin{proof}
We want to apply Theorems~\ref{thm:fbp_visit_lower} and~\ref{thm:fbp_visit_upper} simultaneously.  We partition the search space in the canonical way such that, for all $i \leq n$, $A_i$ contains the set of all search points with fitness $i$. Now we need a precise result for the probability to leave a level and for the probability to visit a level.

First, we consider the probability $p_i$ to leave a given level $i < n$. Suppose the algorithm has a current search point in $A_i$, so it has $i$ leading $1$s and then a $0$. The algorithm leaves level $A_i$ now if and only if it flips the first $0$ of the bit string (probability of $p$) and no previous bits (probability $(1-p)^i$). Hence, $p_i = p(1-p)^i$.

Next we consider the probability $v_i$ to visit a level $i$. We claim that it is exactly $1/2$, following reasoning given in several places before~\cite{DrosteJW02,Sudholt13}. We want to use Lemma~\ref{lem:visitprob} and its analog for upper bounds. Let $i$ be given. For the initial search point, if it is at least on level $i$ (the condition considered by the lemma), the individual is on level $i$ if and only if the $i+1$st bit is a $0$, so exactly with probability $1/2$ as desired for both bounds. Before an individual with at least $i$ leading $1$s is created, the bit at position $i+1$ remains uniformly random (this can be seen by induction: it is uniform at the beginning and does not experience any bias in any iteration while no individual with at least $i$ leading $1$ is created). Once such an individual is created, if the bit at position $i+1$ is $1$, the level $i$ is skipped, otherwise it is visited. Thus, the algorithm skips level $i$ with probability exactly $1/2$, giving $v_i = 1/2$. With these exact values for the $p_i$ and $v_i$, Theorems~\ref{thm:fbp_visit_lower} and~\ref{thm:fbp_visit_upper} immediately yield the claim.
\end{proof}

By computing the geometric series in Theorem~\ref{thm:LO}, we obtain as a (well-known) corollary that the \oea with the classic mutation rate $p = 1/n$  optimizes \leadingones in an expected run time of $n^2\frac{e-1}{2}(1\pm o(1))$.

\section{A Tight Lower Bound for \onemax}
\label{sec:onemax}

In this section, as a first real example of the usefulness of our general method, we prove a lower bound for the run time of the \oea with standard mutation rate $p=\frac 1n$ on \onemax, which is only by an additive term of order $O(n)$ below the upper bound following from the classic fitness level method. This is tighter than the best gap of order $O(n \log\log n)$ proven previously with fitness level arguments. Moreover, our lower bound is the tightest lower bound apart from the significantly more complicated works that determine the run time precise apart from $o(n)$ terms. We defer a detailed account of the literature together with a comparison of the methods to Section~\ref{ssec:onemaxlit}.

We recall that the fitness levels of the \onemax function are given by	 
\[A_i \coloneqq \{x \in \{0,1\}^n \mid \OM(x) = i\}, i \in [0..n].\]
We use the notation $A_{\ge i} \coloneqq \bigcup_{j = i}^n A_j$ and $A_{\le i} \coloneqq \bigcup_{j = 0}^i A_j$ for all $i \in [0..n]$ as defined above for fitness-based partitions, but with the appropriate bounds $0$ and $n$ instead of $1$ and $m$.

We denote by $T_{k,\ell}$ the expected number of iterations the $\oea$, started with a search point in $A_k$, takes to generate a search point in $A_{\ge \ell}$. We further denote by $T_{\rand,\ell}$ the expected number of iterations the \oea started with a random search point takes to generate a solution in $A_{\ge \ell}$. These notions extend previously proposed fine-grained run time notions: $T_{\rand,\ell}$ is the fixed target run time first proposed in~\cite{DoerrJWZ13} as a technical tool and advocated more broadly in~\cite{BuzdalovDDV20}. The time $T_{k,n}$ until the optimum is found when starting with fitness $k$ was investigated in~\cite{AntipovBD20ppsn} when $k > n/2$, that is, when starting with a better-than-average solution. We spare the details and only note that such fine-grained complexity notions (which also include the fixed-budget complexity proposed in~\cite{JansenZ14}) have given a much better picture on how to use EAs effectively than the classic run time $T_{\rand,n}$ alone. In particular, it was observed that different parameters or algorithms are preferable when not optimizing  until the optimum or when starting with a good solution.

For all $k, \ell \in [0..n]$, we denote by $p_{k,\ell}$ the probability that standard bit mutation with mutation rate $p = \frac 1n$ creates an offspring in $A_\ell$ from a parent in $A_k$. We also write $p_{k,\ge \ell} := \sum_{j = \ell}^n p_{k,j}$ to denote the probability to generate an individual in $A_{\ge \ell}$ from a parent in $A_k$. Then $p_i := p_{i, \ge i+1}$ is the probability that the \oea optimizing \onemax leaves the $i$-th fitness level.

\subsection{Upper and Lower Bounds Via Fitness Levels}

Using the notation just introduced, the classic fitness level method (see Theorem~\ref{thm:classic} and note that the fitness of the parent individuals describes a non-decreasing level process with state change probabilities $p_i$) shows that 
\[
T_{k,\ell} \le \sum_{i = k}^{\ell-1} \frac 1 {p_{i}} =: \tilde T_{k,\ell}.
\]

To prove a nearly matching lower bound employing our new methods, we first analyze the probability that the \oea optimizing \onemax skips a particular fitness level. Note that if $q_i$ is the probability to skip the $i$-th fitness level, then $v_i := 1 - q_i$ is the probability to visit the $i$-th level as used in Theorem~\ref{thm:fbp_visit_lower}.

\begin{lemma}\label{lmiss}
  Let $i \in [0..n]$. Consider a run of the \oea with mutation rate $p = \frac 1n$ on the \onemax function started with a (possibly random) individual $x$ with $\onemax(x) < i$. Then the probability $q_i$ that during the run the parent individual never has fitness $i$ satisfies \[q_i \le \frac{n-i}{n(1-\frac 1n)^{i-1}}.\]
\end{lemma}

\begin{proof}
  Since we assume that we start below fitness level $i$, by Lemma~\ref{lem:visitprob} (and using the notation from that lemma for a moment) we have
	\begin{align*}
	q_i &\le \max\{\Pr[x \to A_{\ge i+1} \mid x \to A_{\ge i}] \mid \onemax(x) < i\}\\
	& \le \max_{k \in [0..i-1]} \frac{p_{k,\ge i+1}}{p_{k, \ge i}}.
	\end{align*}
	Hence it suffices to show that $\frac{p_{k,\ge i+1}}{p_{k, \ge i}} \le \frac{n-i}{n(1-\frac 1n)^{i-1}}$ for all $k \in [0..i-1]$, and this is what we will do in the remainder of this proof.
	
  Let us, slightly abusing the common notation, write $\Bin(m,p)$ to denote a random variable following a binomial law with parameters $m$ and~$p$. Let $k, \ell \in \N$ with $k \le \ell$. Noting that the only way to generate a search point in $A_\ell$ from some $x \in A_k$ is to flip, for some $j \in [\ell-k..\min\{n-k,\ell\}]$, exactly $j$ of the $n-k$ zero-bits of $x$ and exactly $j - (\ell-k)$ of the $k$ one-bits, we easily obtain the well-known fact that
	\begin{align*}
	p_{k,\ell} 
	&= \sum_{j = \ell-k}^{\min\{n-k,\ell\}} \Pr[\Bin(n-k,p) = j] \Pr[\Bin(k,p) = j - (\ell-k)]\\
	&= \sum_{j = \ell-k}^{\min\{n-k,\ell\}} \binom{n-k}{j} \binom{k}{j - (\ell-k)} p^{2j - \ell + k} (1-p)^{n - 2j + \ell - k}.
	\end{align*}
  Since $p = \frac 1n$, the mode of $\Bin(n-k,p)$ is at most $1$. Since the binomial distribution is unimodal, we conclude that $\Pr[\Bin(n-k,p) = j] \le \Pr[\Bin(n-k,p) = \ell-k]$ for all $j \ge \ell-k$. Consequently, the first line of the above set of equations gives 
	\begin{align*}
	p_{k,\ell} &\le \Pr[\Bin(n-k,p) = \ell-k] \Pr[\Bin(k,p) \in [0..\min\{n-\ell,k\}]] \\
	&\le \Pr[\Bin(n-k,p) = \ell-k]
	\end{align*}
	and thus	
	\begin{equation}
	p_{k, \ge\ell} \le \Pr[\Bin(n-k,p) \ge \ell-k].\label{eq:pbin}
	\end{equation}
  
	%
  We recall that our target is to estimate $\frac{p_{k,\ge i+1}}{p_{k, \ge i}}$ for all $k \in [0..i-1]$. By~\eqref{eq:pbin}, we have 
\begin{align*}
p_{k, \ge i+1} &\le \Pr[\Bin(n-k,p) \ge i+1 - k] \\
&\le \frac{(i+1-k)(1-p)}{i+1-k-(n-k)p} \Pr[\Bin(n-k,p) = i+1 - k],
\end{align*} 
where the last estimate is~\cite[equation following Lemma~1.10.38]{Doerr20bookchapter}. We also have $p_{k,\ge i} \ge p_{k,i} \ge (1-p)^k \Pr[\Bin(n-k,p)=i-k]$. Hence from 
\begin{align*}
\frac{\Pr[\Bin(n-k,p) = i+1 - k]}{\Pr[\Bin(n-k,p)=i-k]} 
&= \frac{\binom{n-k}{i+1-k} p^{i+1-k} (1-p)^{n-k-(i+1-k)}}{\binom{n-k}{i-k} p^{i-k} (1-p)^{n-k-(i-k)}} \\
&= \frac{(n-i)p}{(i+1-k)(1-p)}
\end{align*}
we conclude
\begin{align*}
  \frac{p_{k,\ge i+1}}{p_{k, \ge i}} &\le \frac{(i+1-k)(1-p)}{i+1-k-(n-k)p} \frac{(n-i)p}{(i+1-k)(1-p)^{k+1}}\\
  & \le \frac{n-i}{n(i-k)(1-\frac 1n)^k},
 \end{align*}
using again that $p = \frac 1n$. For $k \in [0..i-1]$, this expression is maximal for $k = i-1$, giving that $q_i \le \frac{n-i}{n(1-\frac 1n)^{i-1}}$ as claimed.
\end{proof}

With this estimate, we can now easily give a very tight lower bound on the run time of the \oea on \onemax. 

\begin{theorem}\label{thm:onemax1}
  Let $k , \ell \in [0..n]$ with $k < \ell$. Then the expected number $T_{k,\ell}$ of iterations the \oea optimizing \onemax and initialized with any search point $x$ with $\onemax(x) = k$ takes to generate a search point $z$ with fitness $\onemax(z) \ge \ell$ is at least 
	\[
	T_{k,\ell} \ge \tilde T_{k,\ell} - (\ell-k-1) e (e-1) \exp\left(\frac{k}{n-1}\right),
	\]  
	where $\tilde T_{k,\ell}$ is the upper bound stemming from the fitness level method as defined at the beginning of this section. This lower bound holds also for $T_{k',\ell}$ with $k' \le k$, that is, when starting with a search point $x$ with $\onemax(x) \le k$.
\end{theorem}

\begin{proof}
We use our main result Theorem~\ref{thm:fbp_visit_lower}. We note first that when assuming that the level process regarded in Theorem~\ref{thm:fbp_visit_lower} starts on level~$k'$, then the expected time for it to reach level $\ell$ or higher is at least $\sum_{i=k'}^{\ell-1} \frac{v_i}{p_i}$. This follows immediately from the proof of the theorem or by applying the theorem to the level process $(X_t')$ defined by $X'_t = \min\{\ell, X_t\} - k'$ for all~$t$. 

Consider now a run of the \oea on the \onemax function started with an initial search point $x_0$ such that $k' = \onemax(x_0) \le k$. Denote by $x_t$ the individual selected in iteration $t$ as future parent. Then $X_t = \onemax(x_t)$ defines a level process. As before, we denote the probabilities to visit level $i$ by $v_i$, to not visit it by $q_i = 1 - v_i$, and to leave it to a higher level by $p_i$. Using our main result and the elementary argument above, we obtain an expected run time of
\begin{align*}
  E[T_{k',\ell}] 
	\ge \sum_{i = k'}^{\ell-1} \frac{v_i}{p_{i}} 
  \ge \sum_{i = k}^{\ell-1} \frac{v_i}{p_{i}} 
  \ge \sum_{i = k}^{\ell-1} \frac{1}{p_{i}} - \sum_{i=k+1}^{\ell-1} \frac{q_i}{p_{i}}.
\end{align*}  
We note that the first expression is exactly the upper bound $\tilde T_{k,\ell}$ stemming from the classic fitness level method. We estimate the second expression. We have 
\begin{equation}
  p_i = p_{i,\ge i+1} \ge p_{i,i+1} \ge (1 - \tfrac 1n)^{n-1} \tfrac{n-i}{n}, \label{eq:lbpi}
\end{equation}
where the last estimate stems from regarding only the event that exactly one missing bit is flipped. Together with the estimate $q_i \le \frac{n-i}{n(1-\frac 1n)^{i-1}}$ from Lemma~\ref{lmiss}, we compute
\begin{align}
  \sum_{i=k+1}^{\ell-1}  \frac{q_i}{p_{i}} 
  &\le \sum_{i=k+1}^{\ell-1} \frac{n-i}{n(1-\frac 1n)^{i-1}} \frac{n}{(n-i) (1 - \frac 1n)^{n-1}} \nonumber\\
  &= \sum_{i=k+1}^{\ell-1} \left(1 + \frac 1 {n-1}\right)^{n+i-2} \nonumber\\
	&= \left(1 + \frac 1 {n-1}\right)^{n+k-1} \,\,  \sum_{j=0}^{\ell-k-2} \left(1 + \frac 1 {n-1}\right)^j\nonumber\\
	&= \left(1 + \frac 1 {n-1}\right)^{n+k-1} \frac{\left(1 + \frac 1 {n-1}\right)^{\ell-k-1} - 1}{\left(1 + \frac 1 {n-1}\right) - 1}\nonumber\\
	&= \left(1 + \frac 1 {n-1}\right)^{n+k-1} (n-1) \left(\left(1 + \frac 1 {n-1}\right)^{\ell-k-1} - 1\right)\nonumber\\
	&\le  (n-1) \exp\left(\frac{n+k-1}{n-1}\right) \left(\exp\left(\frac{\ell-k-1}{n-1}\right) - 1\right)\label{eq:onemax1sharp}\\
	& = (n-1) e \exp\left(\frac{k}{n-1}\right) \left(\exp\left(\frac{\ell-k-1}{n-1}\right) - 1\right)\nonumber\\
	&\le  (n-1) (e-1) e \exp\left(\frac{k}{n-1}\right) \frac{\ell-k-1}{n-1},	\nonumber
\end{align}
where the estimate in \eqref{eq:onemax1sharp} uses the well-known inequality $1+r \le e^r$ valid for all $r \in \R$ and the last estimate exploits the convexity of the exponential function in the interval $[0,1]$, that is, that $\exp(\alpha) \le 1 + \alpha(\exp(1)-\exp(0))$ for all $\alpha \in [0,1]$.
\end{proof}

The result above shows that the classic fitness level method and our new lower bound method can give very tight run time results. We note that the difference $\delta_{k,\ell} = (\ell-k-1) e (e-1) \exp\left(\frac{k}{n-1}\right)$ between the two fitness level estimates is only of order $O(\ell - k)$, in particular, only of order $O(n)$ for the classic run time $T_{\rand,n}$, which itself is of order $\Theta(n \log n)$. Hence here the gap is only a term of lower order. 


\subsection{Estimating the Fitness Level Estimate $\tilde T_{k,\ell}$}

To make our results above meaningful, it remains to analyze the quantity $\tilde T_{k,\ell}= \sum_{i=k}^{\ell-1} 1/p_{i}$, which is the estimate from the classic fitness level method. 

Here, again, it turns out that upper bounds tend to be easier to obtain since they require a lower bound for the $p_{i}$, for which the estimate $p_{i} \ge (1-\tfrac 1n)^{n-1} \frac{n-i}{n}$ from~\eqref{eq:lbpi} usually is sufficient. To ease the presentation, let us use the notation $e_n = (1 - \tfrac 1n)^{-(n-1)}$ and note that $e (1-\frac 1n) \le e_n \le e$, see, e.g., \cite[Corollary 1.4.6]{Doerr20bookchapter}. With this notation, the lower bound~\eqref{eq:lbpi} gives the upper bound 
\begin{equation}
\tilde T_{k,\ell} \le e_n n \sum_{i = k}^{\ell-1} \frac{1}{n-i} =: \tilde T_{k,\ell}^+.\label{eq:levub}
\end{equation}

To prove a lower bound, we observe that 
\[
p_{i} = \sum_{d = 1}^{n-i} \Pr[\Bin(n-i,p) = d] \Pr[\Bin(i,p) < d].
\] 
We can thus estimate 
\begin{align}
p_{i} &\le \Pr[\Bin(n-i,p) = 1] \Pr[\Bin(i,p) =0] + \Pr[\Bin(n-i,p) \ge 2] \nonumber\\
&\le \left(1-\frac 1n\right)^{n-1} \frac{n-i}{n} + \frac{(n-i)(n-i-1)}{2n^2}, \label{eq:pub}
\end{align}
where the last inequality follows from the estimate $\Pr[\Bin(n,p) \ge k] \le \binom{n}{k} p^k$, see, e.g.,~\cite[Lemma~3]{GiessenW17} or \cite[Lemma~1.10.37]{Doerr20bookchapter}. We note that the first summand in~\eqref{eq:pub} is exactly our lower bound~\eqref{eq:lbpi} for $p_{i}$, so it is the second term that determines the slack of our estimates. We estimate coarsely
\begin{align*}
\frac 1 {p_{i}} 
& \ge \left(\left(1-\frac 1n\right)^{n-1} \frac{n-i}{n} + \frac{(n-i)(n-i-1)}{2n^2}\right)^{-1} \\
& = \frac{2 e_n n^2}{2n(n-i) + e_n(n-i)(n-i-1)}\\
& = \frac{e_n n}{n-i} - \frac{e_n^2 n (n-i-1)}{(n-i)(2n + e_n(n-i-1))} \ge \frac{e_n n}{n-i} - \frac 12 e_n^2. 
\end{align*}

Summing over the fitness levels, we obtain
\begin{align}
  \tilde T_{k,\ell} 
	& = \sum_{i=k}^{\ell-1} \frac 1 {p_{i}}\nonumber\\
	& \ge \sum_{i=k}^{\ell-1} \left(\frac{e_n n}{n-i} -  \frac 12 e_n^2 \right)\nonumber\\
	&= \tilde T_{k,\ell}^+ - \tfrac 12 e_n^2 (\ell-k)  =: \tilde T_{k,\ell}^-.\label{eq:tildetlb}
\end{align}
	
	We note that our upper and lower bounds on $\tilde T_{k,\ell}$ deviate only by $\tilde T_{k,\ell}^+ - \tilde T_{k,\ell}^- = \frac 12 e_n^2 (\ell-k)$. Together with Theorem~\ref{thm:onemax1}, we have proven the following estimates for $T_{k,\ell}$, which are tight apart from a term of order $O(\ell-k)$.
	\begin{theorem}\label{thm:onemax2}
	  The expected number of iterations the \oea optimizing \onemax, started with a search point of fitness $k$, takes to find a search point with fitness $\ell$ or larger, satisfies
		\begin{align*}
		&	e_n n \sum_{i = n-\ell+1}^{n-k} \frac{1}{i} \,-\,  (\ell-k-1) e (e-1) \exp\left(\frac{k}{n-1}\right) - \frac 12 e_n^2 (\ell-k) \\
		& \le T_{k,\ell} \le \\
		&e_n n \sum_{i = n-\ell+1}^{n-k} \frac{1}{i}\,,
		\end{align*}
		where $e_n := (1-\frac 1n)^{-(n-1)}$.
	\end{theorem}
	We recall from above that $e (1-\frac 1n) \le e_n \le e$. We add that for $\ell<n$, the sum $\sum_{i = n-\ell+1}^{n-k} \frac{1}{i}$ is well-approximated by $\ln(\frac{n-k}{n-\ell})$, e.g., $\ln(\frac{n-k}{n-\ell}) -1 < \sum_{i = n-\ell+1}^{n-k} \frac{1}{i} < \ln(\frac{n-k}{n-\ell})$ or $\sum_{i = n-\ell+1}^{n-k} \frac{1}{i} = \ln(\frac{n-k}{n-\ell}) - O(\frac{1}{n-\ell})$, see, e.g.,~\cite[Section 1.4.2]{Doerr20bookchapter}. For $\ell = n$, we have $\ln(n-k) < \sum_{i = n-\ell+1}^{n-k} \frac{1}{i} \le \ln(n-k)+1$ and $\sum_{i = n-\ell+1}^{n-k} \frac{1}{i} = \ln(n-k)+O(\frac{1}{n-k})$.
	
	When starting the \oea with a random initial search point, the following bounds apply.
	
\begin{theorem}
  There is an absolute constant $K$ such that the expected run time $T = T_{\rand,n}$ of the \oea with random initialization on \onemax satisfies
	\[
e_n n \sum_{i = 1}^{\lceil n/2 \rceil} \frac{1}{i} - 4.755 n - K \le T \le e_n n \sum_{i = 1}^{\lceil n/2 \rceil} \frac{1}{i} + K.
	\]
	In particular, 
	\[
	en \ln(n) - 4.871n - O(\log n) \le T \le e n \ln(n) - 0.115 n + O(1).
	\]
\end{theorem}
  
\begin{proof}
  By~\cite[Theorem~2]{DoerrD16}, the expected run time of the \oea with random initialization on \onemax differs from the run time when starting with a search point on level $A_M$, $M \coloneqq \lfloor n/2 \rfloor$, by at most a constant. Hence we have $T \le T_{M,n} + O(1) \le \tilde T^+_{M,n} + O(1) = e_n n \sum_{i = 1}^{\lceil n/2 \rceil} \frac{1}{i} + O(1)$ by Theorem~\ref{thm:onemax2}. 
	
	For the lower bound, we use Equation~\eqref{eq:onemax1sharp} in the proof of Theorem~\ref{thm:onemax1}, which is slightly tighter than the result stated in the theorem itself. Together with~\eqref{eq:tildetlb}, we estimate 
	\begin{align*}
	T &\ge T_{M,n} -O(1)\\ 
	&\ge \tilde T_{M,n} -  (n-1) \exp(\tfrac{n + M - 1}{n-1}) (\exp(\tfrac{n - M - 1}{n-1}) - 1) - O(1) \\
	&\ge \tilde T_{M,n}^+ - \tfrac 12 e_n^2 (n - M) - n e^{1.5} (e^{0.5} - 1) - O(1)\\
	&= \tilde T_{M,n}^+ - \tfrac 14 e^2 n - n  (e^2 - e^{1.5}) - O(1) \\
	&= \tilde T_{M,n}^+ - n  (\tfrac 54 e^2 - e^{1.5}) - O(1) \ge \tilde T_{M,n}^+ - 4.755 n - O(1).
	\end{align*}
	The second set of estimates stems from noting that $\tilde T_{M,n}^+ = e_n n \sum_{i = 1}^{\lceil n/2 \rceil} \frac{1}{i} = e_n n (\ln(\lceil n/2 \rceil) + \gamma \pm O(\frac 1n)) = e (1 - O(\frac 1n)) n (\ln n - \ln 2  + \gamma \pm O(\frac 1n))$, where $\gamma =  0.5772156649\dots$ is the Euler-Mascheroni constant.
\end{proof}	

Let us comment a little on the tightness of our result. Due to the symmetries in the \onemax process, the probability to leave the $i$-th fitness level is independent of the particular search point $x \in A_i$ the current parent is equal to. Consequently, in principle, Theorems~\ref{thm:fbp_visit_upper} and~\ref{thm:fbp_visit_lower} give the exact bound \[E[T] = \sum_{k = 0}^{n-1} 2^{-n} \binom{n}{k} \sum_{i = k}^{n-1} \frac{v_{i|k}}{p_i},\]
where $v_{i|k}$ denotes the probability that the process started on level $k$ visits level $i$.

The reason why we cannot avoid a gap of order $\Theta(n)$ in our bounds is that computing the $v_{i|k}$ and $p_i$ precisely is very difficult. Let us regard the $v_{i|k}$ first. It is easy to see that states $i$ with $k < i \le (1-\eps) n$, $\eps$ a positive constant, have a positive chance of not being visited: By Lemma~\ref{lmiss}, with probability $\Omega(1)$ level $i-1$ is visited and from there, again with probability $\Omega(1)$, a two-bit flip occurs that leads to level $i+1$. Since with constant probability the last level visited below level $i$ is not $i-1$, and since skipping level $i$ conditional on the last level below $i$ being at most $i-2$ is, by a positive constant, less likely that skipping level $i$ when on level $i-1$ before (that is, $\frac{p_{i-2,\ge i+1}}{p_{i-2,\ge i}} \le \frac{p_{i-1,\ge i+1}}{p_{i-1,\ge i}} - \Omega(1)$, we omit a formal proof of this statement), our estimate $q_{i|k} \le \max_{j \in [k..i-1]} \frac{p_{j,\ge i+1}}{p_{j,\ge i}}$ already leads to a constant factor loss in the estimate of the $q_i$, which translates into a $\Theta(n)$ contribution to the gap of our lower bound from the truth. To overcome this, one would need to compute $q_{i|k} = \sum_{j = k}^{i-1} Q_{j|k} \frac{p_{j,\ge i+1}}{p_{j,\ge i}}$ precisely, where $Q_{j|k}$ is the probability that level $j$ is the highest level visited below $i$ in a process started on level~$k$. This appears very complicated. 

The second contribution to our $\Theta(n)$ gap is the estimate of $p_i$. We need a lower bound on $p_i$ both in the estimate of the run time advantage due to not visiting all levels (see Equation~\eqref{eq:onemax1sharp}) and in the estimate of the run time estimate stemming from the fitness level method~\eqref{eq:levub}. Since the $q_i$ are $\Omega(1)$ when $i \le (1-\eps)n$, a constant-factor misestimation of the $p_i$ leads to a $\Theta(n)$ contribution to the gap. Unfortunately, it is hard to avoid a constant-factor misestimation of the $p_i$, $i \le (1-\eps)n$. Our estimate $p_i \ge (1-\frac 1n)^{n-1} \frac{n-i}{n}$ only regards the event that the $i$-th level is left (to level $i+1$) by flipping exactly one zero-bit into a one-bit. However, for each constant $j$ the event that level $i+1$ is reached by flipping $j+1$ zero-bits and $j$ one-bits has a constant probability of appearing. Moreover, for each constant $j$ the event that level $i$ is left to level $i+j$ also has a constant probability. For these reasons, a precise estimate of the $p_i$ appears rather tedious. 

In summary, we feel that our method quite easily gave a run time estimate precise apart from terms of order $O(n)$, but for more precise results drift analysis~\cite{Lengler20bookchapter} might be the better tool (though still the relatively precise estimate of the expected progress from a level $i \le (1-\eps)n$, which will necessarily be required for such an analysis, will be difficult to obtain).

\subsection{Comparison with the Literature}\label{ssec:onemaxlit}

We end this section by giving an overview on the previous works analyzing the run time of the \oea on \onemax and comparing them to our result. Some of the results described in the following, in particular, Sudholt's lower bound~\cite{Sudholt13}, were also proven for general mutation rates $p$ instead of only $p = \frac 1n$. To ease the comparison with our result, we only state the results for the case that $p = \frac 1n$. We note that with our method we could also have analysed broader ranges of mutation rates. The resulting computations, however, would have been more complicated and would have obscured the basic application of our method.

To the best of our knowledge, the first to state and rigorously prove a run time bound for this problem was Rudolph in his dissertation~\cite[p.~95]{Rudolph97}, who showed that $T = T_{\rand,n}$ satisfies $E[T] \le (1-\frac 1n)^{n-1} n \sum_{i=1}^{n} \frac{1}{i}$, which is exactly the upper bound $\tilde T^+_{0,n}$ from the fitness level method and from only regarding the events that levels are left via one-bit flips. A lower bound of $n \ln(n) - O(n \log\log n)$ was shown in~\cite{DrosteJW98ecj} for the optimization of a general separable function with positive weights when starting in the search point $(0, \dots, 0)$. From the proof of this result, it is clear that it holds for any pseudo-Boolean function with unique global optimum $(1, \dots, 1)$. This lower bound builds on the argument that each bit needs to be flipped at least once in some mutation step. It is not difficult to see that the expected time until this event happens is indeed $(1 \pm o(1)) n \ln n$, so this argument is too weak to make the leading constant of $E[T]$ precise. 

Only a very short time after these results and thus quite early in the young history of run time analysis of evolutionary algorithms, Garnier, Kallel, and Schoenauer~\cite{GarnierKS99} showed that $E[T] = en\ln(n) + c_1 n + o(n)$ for a constant $c_1 \approx -1.9$, however, the completeness of their proof has been doubted in~\cite{HwangPRTC18}. Since at that early time precise run time analyses were not very popular, it took a while until Doerr, Fouz, and Witt~\cite{DoerrFW10} revisited this problem and showed with $E[T] \ge (1-o(1)) e n \ln(n)$ the first lower bound that made the leading constant precise. Their proof used a variant of additive drift from~\cite{Jagerskupper07} together with the potential function $\ln(Z_t)$, where $Z_t$ denotes the number of zeroes in the parent individual at time $t$. Shortly later, Sudholt~\cite{Sudholt10} (journal version~\cite{Sudholt13}) used his fitness level method for lower bounds to show $E[T] \ge en \ln(n) - 2n\log\log n - 16n$. That the run time was $E[T] = en\ln(n) - \Theta(n)$ was proven first in~\cite{DoerrFW11}, where an upper bound of $en\ln(n) - 0.1369n + O(1)$\footnote{The constant $0.1369$ was wrongly stated as $0.369$ as pointed out in~\cite{LehreW14}} was shown via variable drift for upper bounds~\cite{MitavskiyRC09,Johannsen10} and a lower bound of $E[T] \ge en\ln(n) - O(n)$ was shown via a new variable drift theorem for lower bounds on hitting times. An explicit version of the lower bound of $en\ln(n) - 7.81791n - O(\log n)$ and an alternative proof of the upper bound $en\ln(n) - 0.1369n + O(1)$ was given in~\cite{LehreW14} via a very general drift theorem.

The final answer to this problem was given in an incredibly difficult work by Hwang, Panholzer, Rolin, Tsai, and Chen~\cite{HwangPRTC18} (see~\cite{HwangW19} for a simplified version), who showed $E[T] = en\ln(n) + c_1 n + \frac 12 e \ln(n) + c_2 + O(n^{-1} \log n)$ with explicit constants $c_1 \approx -1.9$ and $c_2 \approx 0.6$. 

In the light of these results, we feel that our proof of an $en\ln(n) \pm O(n)$ bound is the first simple proof a run time estimate of this precision for this problem. Interestingly, our explicit lower bound $en\ln(n) - 4.871n - O(\log n)$ is even a little stronger than the bound $en\ln(n) - 7.81791n - O(\log n)$ proven with drift methods in~\cite{LehreW14}.

\section{Jump Functions}\label{sec:jump}

In this section, we regard jump functions, which comprise the most intensively studied benchmark in the theory of randomized search heuristics that is not unimodal and which has greatly aided our understanding how different heuristics cope with local optima~\cite{DrosteJW02,JansenW02,Lehre10,DoerrLMN17,CorusOY17,CorusOY18fast,DangFKKLOSS16,DangFKKLOSS18,WhitleyVHM18,LissovoiOW19,RoweA19,Doerr20gecco,AntipovDK20,AntipovD20ppsn,AntipovBD21gecco,RajabiW20,RajabiW21evocop,RajabiW21gecco,DoerrZ21aaai,BenbakiBD21}.

For all representation lengths $n$ and all $k \in [1..n]$, the \emph{jump function} with jump size $k$ is defined by 
\[
\jump_{n,k}(x)=\left\{\begin{array}{ll}
\|x\|_{1}+k & \text { if }\|x\|_{1} \in[0 . . n-k] \cup\{n\}, \\
n-\|x\|_{1} & \text { if }\|x\|_{1} \in[n-k+1 . . n-1],
\end{array}\right.
\]
for all $x \in \{0,1\}^n$. Jump functions have a fitness landscape isomorphic to $\onemax$, except on the fitness valley or gap 
\[
G_{n,k}:=\left\{x \in\{0,1\}^{n} \mid n-k<\|x\|_{1}<n\right\},
\]
where the fitness is low and deceptive (pointing away from the optimum).

For simple elitist heuristics, not surprisingly, the time to find the optimum is strongly related to the time to cross the valley of low fitness. For the \oea with mutation rate $\frac 1n$, the probability to generate the optimum from a search point on the local optimum $L = \{x \in \{0,1\}^n \mid \|x\|_1 = n-k\}$ is $p_k = (1-\frac 1n)^{n-k} n^{-k}$, and hence the expected time to cross the valley of low fitness is $\frac 1 {p_k}$. 

The true expected run time deviates slightly from this value, both because some time is spent to reach the local optimum and because the algorithm may be lucky and not need to cross the valley or not in its full width. The first aspect, making additive terms of order at most $O(n \log n)$ more precise, can be treated with arguments very similar to the ones of the previous section, so we do not discuss this here. More interestingly appears the second aspect. In particular for larger values of $k$, the algorithm has a decent chance to start in the fitness valley. It is clear that even when starting in the valley, the deceptive nature of the valley will lead the algorithm rather towards the local optimum. We show now how our argumentation via omitted fitness levels allows to prove very precise bounds with  elementary arguments. In principle, we could also use the our fitness level theorem, but since we shall regard only the single level $N_{n,k} = \{x \in \{0,1\}^n \mid \|x\|_1 \in [0..n-k]\}$, we shall not make this explicit and simply use the classic typical-run argument (that except with some probability $q$, a state is reached from which the expected run time is at least some $t$, and that this gives a lower bound of $(1-q)t$ for the expected run time). 

The two previous analyses of the run time of the \oea on jump functions deal with the problem of starting in the valley in a different manner. In~\cite{DrosteJW02}, it is argued that with probability at least $\frac 12$, the initial search point has at most $\frac n2$ ones. In case the initial search point is nevertheless in the gap region (because $k > \frac n2$), then with high probability a \onemax-style optimization process will reach the local optimum with high probability in time $O(n^2)$ except when in this period the optimum is generated. Since all parent individuals in this period have Hamming distance at least $\frac n2$ from the optimum, the probability for this exceptional event is exponentially small. This argument proves an $\Omega(\frac 1 {p_k})$ bound for the expected run time, and this for all values of $k \ge 2$.  In~\cite{DoerrLMN17}, only the case $k \le \frac n2$ was regarded and it was exploited that in this case, the probability for the initial search point to be in the gap (or the optimum) is only $2^{-n} \binom{n}{\le k-1}$. This gives a lower bound of $\big(1 - 2^{-n} \binom{n}{\le k-1}\big) \frac 1 {p_k}$, which is tight including the leading constant for $k \in [2..\frac n2 - \omega(\sqrt n)]$. 

We now show that estimating the probability of never reaching a search point $x$ with $\|x\|_1 \le n-k$ is not difficult with arguments similar to the ones used in the previous section. We need a slightly different approach since now the probability to skip a fitness level is not maximal when closest to this fitness level (the probability to skip $N_{n,k}$ is maximal when in the lowest fitness level, which is in Hamming distance $k-1$ from $N_{n,k}$). Interestingly, we obtain very tight bounds which could be of some general interest, namely that the probability to never reach a point $x$ with $\|x\| \le n-k$ is $O(\frac 1n)$, when allowing an arbitrary initialization (different from the global optimum), and is only $O(2^{-n})$ when using the usual random initialization. 

\begin{theorem}
  Let $n \in \N$ and $k \in [2..n]$. Consider a run of the \oea with mutation rate $p = \frac 1n$ on the jump function $\jump_{n,k}$. Denote by $N := N_{n,k} = \{x \in \{0,1\}^n \mid \|x\|_1 \in [0..n-k]\}$ the set of non-optimal solutions that lie not in the gap region of the jump function and by $p_k = (1-\frac 1n)^{n-k} n^{-k}$ the probability to generates the optimum from a solution on the local optimum.
	\begin{enumerate}
	\item Assume that the \oea starts with an arbitrary solution different from the global optimum. Then with probability $1 - O(\frac 1n)$, the algorithm reaches a search point in $N$. Consequently, the expected run time is at least $(1 - O(\frac 1n)) p_k^{-1}$.
	\item Assume that the \oea starts with a random initial solution. Then with probability $1 - O(2^{-n})$, the algorithm reaches a search point in $N$. Consequently, the expected run time is at least $(1 - O(2^{-n})) p_k^{-1}$.
	\end{enumerate}
\end{theorem}

\begin{proof}
  Denote by $f$ the jump function $\jump_{n,k}$. We consider the partition of the search space into the fitness levels of the gap as well as $N$ and the optimum. Hence let
	\begin{align*}
	A_j &:= \{x \in \{0,1\}^n \mid f(x) = j\} \mbox{ for } j \in [1..k-1],\\
	A_{k} &:= \{x \in \{0,1\}^n \mid f(x) \in [k..n]\} = N,\\
	A_{k+1} &:= \{(1, \dots, 1)\}. 
  \end{align*}
	Our first claim is that, regardless of the initialization as long as different from the optimum, the probability $q_k$ that the algorithm never has the parent individual in $A_k$ is $O(\frac 1n)$. Since we start the algorithm with a non-optimal search point, the only way the algorithm can avoid $A_k$ is by generating from a parent in $A_j$, $j \in [1..k-1]$, the global optimum. Denote by $r_j$ the probability that the algorithm, if the current search point is in $A_j$, in the remaining run generates the optimum from a search point in $A_j$. Then, as just discussed, by the union bound,
	\begin{equation}
	q_k \le \sum_{j=1}^{k-1} r_j. \label{eq:jumpr}
	\end{equation}
	The probability $r_j$ is exactly the probability that in the iteration in which from a search point in $A_j$ a better individual is generated, this is actually the global optimum. Hence $r_j = \Pr[y = (1, \dots, 1) \mid f(y) > j]$, where $y$ is a mutation offspring generated from a search point in $A_j$. We compute
  \begin{align}
		r_j &= \Pr[y = (1, \dots, 1) \mid f(y) > j] = \frac{\Pr[y = (1, \dots, 1)]}{\Pr[f(y) > j]} \nonumber\\
		&\le \frac{n^{-j}}{(1-\frac 1n)^{n-1} \frac {n-j}n} \le \frac{e}{n^{j-1} (n-j)},\nonumber
	\end{align}	
	where we estimated the probability to generate a search point with fitness better than $j$ by the probability of the event that a single one is flipped into a zero. Consequently, $q_k \le \sum_{j=1}^{k-1} r_j \le \sum_{j=1}^{n-1} \frac{e}{n^j (n-j)} = O(\frac 1n)$.
	
	Once a search point in $A_k$ is reached, the remaining run time dominates a geometric distribution with success probability $p_k = (1-\frac 1n)^{n-k} n^k$, simply because each of the following iterations (before the optimum is found) has at most this probability of generating the optimum; hence the expected remaining run time is at least $\frac 1 {p_k}$. This shows that the expected run time of the \oea started with any non-optimal search point is at least $(1-q_k) \frac 1 {p_k}$. 
	
	For the case of a random initialization, we proceed in a similar manner, but also use the trivial observation that to skip the fitness range $N$ by jumping from $A_j$, $j \in [1..k-1]$, right into the optimum, it is necessary that the algorithm visits $A_j$. To visit $A_j$, it is necessary that the initial search point lies in $A_1 \cup \dots \cup A_j$, which happens with probability $2^{-n} \sum_{i=1}^{k-1} \binom{n}{i}$ only. This, together with the observation that the only other way to avoid $A_k$ is that the initial individual is already the optimum, gives
	\begin{align*}
	q_k &\le \sum_{j=1}^{k-1} \left(2^{-n} \sum_{i=1}^{j} \binom{n}{i}\right) r_j + 2^{-n}.	
	\end{align*}
	Using a tail estimate for binomial distributions (equation (VI.3.4) in~\cite{Feller68}, also to be found as (1.10.62) in~\cite{Doerr20bookchapter}), we bound $\sum_{i=1}^{j} \binom{n}{i} \le 1.5 \binom{n}{j}$ for all $j \le \frac 14 n$. We also note from~\eqref{eq:jumpr} that $r_j \le 4 n^{-j}$ in this case. For $j \ge \frac 14 n$, we trivially have $\sum_{i=1}^{j} 2^{-n} \binom{n}{i} r_j \le r_j \le e r^{-(j-1)}$. Consequently, 
	\begin{align*}
	q_k &\le \sum_{j=1}^{n-1} \left(2^{-n} \sum_{i=1}^{j} \binom{n}{i}\right) r_j + 2^{-n}\\
	&\le 2^{-n} 1.5 \sum_{j = 1}^{\lfloor n/4 \rfloor} \binom{n}{j} r_j + \sum_{j = \lceil n/4 \rceil}^{n-1} e n^{-(j-1)}  + 2^{-n}\\
	&\le 2^{-n} 1.5 \sum_{j = 0}^{\lfloor n/4 \rfloor} \frac{n^j}{j!} 4 n^{-j} + O(n^{-(n/4)+1}) \le 2^{-n} 6 e  + O(n^{-(n/4)+1}).
	\end{align*}
	Hence, as above, the expected run time is at least $(1-q_k) \frac 1 {p_k} = (1 - O(2^{-n})) \frac 1 {p_k}$.	
\end{proof}

We note that the $O(\frac 1n)$ term in the bound for arbitrary initialization cannot be avoided in general, simply because when starting with a search point that is a neighbor of the optimum, the first iteration with probability at least $\frac 1{en}$ generates the optimum. The $O(2^{-n})$ term in the bound for random initialization is apparently necessary because with probability $2^{-n}$.

We also note that we did not optimize the implicit constants in the $O(\frac 1n)$ and $O(2^{-n})$ term. With more care, these could be replaced by $(1 + o(1)) \frac{1}{e-1} \frac 1n$ and $(1+o(1)) \frac{e}{e-1} 2^{-n}$, respectively.

\section{A Bound for Long $k$-Paths}
\label{sec:longKPaths}

Long $k$-paths, introduced in \cite{Rudolph96}, have been studied in various places; we point the reader to \cite{Sudholt09} for a discussion, which also contains the formalization that we use. A lower bound for long $k$-paths using FLM with viscosities was given in~\cite{Sudholt13}.

We use \cite[Lemma~3]{Sudholt09} (phrased as a definition below) and need to know no further details about what a long $k$-path is. In fact, our proof uses all the ideas of the proof of \cite{Sudholt13}, but cast in terms of our FLM with visit probabilities, which, we believe, makes the proof simpler and the core ideas given by \cite{Sudholt13} more prominent. Note that \cite{Sudholt13} first needs to extend the FLM with viscosities by introducing an additional parameter before it is applicable in this case.

\begin{definition}
Let $k, n$ be given such that $k$ divides $n$. A \emph{long $k$-path} is function $f: \{0,1\}^n \rightarrow \R$ such that
\begin{itemize}
	\item The $0$-bit string has a fitness of $0$;  there are $m = k2^{n/k} - k$ bit strings of positive fitness, and all these values are distinct; all other bit strings have negative fitness. We call the bit strings with non-negative fitness as being \emph{on the path} and consider them ordered by fitness (this way we can talk about the ``next'' element on the path and similar).
	\item For each bit string with non-negative fitness and each $i < k$, the bit string with $i$-next higher fitness is exactly a Hamming distance of $i$ away.
	\item For each bit string with non-negative fitness and each $i\geq k$, the bit string with $i$-next higher fitness is at least a Hamming distance of $k$ away.
\end{itemize}
\end{definition}
For an explicit construction of a long $k$-path, see~\cite{DrosteJW02,Sudholt09}. The long $k$-paths are designed such that optimization proceeds by following the (long) path and true shortcuts are unlikely, since they require jumping at least $k$.

The following lower bound for optimizing long $k$-paths with the \oea is given in \cite{Sudholt13}. Note that $n$ is the length of the bit strings, $m$ is the length of the path and $p$ is the mutation rate.
\begin{equation}
m \; \frac{1-2p}{p(1-p)^{n}} \; \frac{1-2p}{1-p} \; \left(1- \left( \frac{p}{1-p}\right)^{k}\right)^m. \label{eq:sud}
\end{equation}
We want to show here that we can derive the essentially same bound with the same ideas but less technical details.

Note that the lower bound given in \cite{Sudholt13} is only meaningful for $k \geq \sqrt{n/\log(1/p)}$, as the last term of the bound would otherwise be close to~$0$:
\begin{align*}
\left(1- \left(\frac{p}{1-p}\right)^k\right)^m &\leq \left(1- p^k\right)^m \leq \exp(-mp^k) \\
&\leq \exp(-2^{n/k}p^k) = \exp(-2^{n/k-k\log(1/p)}).
\end{align*}
We have that $n/k-k\log(1/p)$ is positive if and only if $n/\log(1/p) \geq k^2$.

In fact, if $k = \omega\left(\sqrt{n/\log(1/p)}\right)$, we have
\begin{align*}
\left(1- \left(\frac{p}{1-p}\right)^k\right)^m 
 & \geq \left(1- \left(2p\right)^k\right)^m\\
 & \geq 1 - m(2p)^k\\
 & \geq 1-2^{3n/k-k\log(1/p)}\\
 & = 1-2^{\sqrt{n}o(\log(1/p))-\sqrt{n}\omega(\log(1/p))}\\
 & = 1-2^{-\sqrt{n}\omega(\log(1/p))}\\
 & \geq 1 - 2^{-\sqrt{n}}.
\end{align*}
This also entails 
$$
p \leq \exp(-n/k^2).
$$

With our fitness level method, we obtain the following lower bound. It differs from Sudholt's bound~\eqref{eq:sud} by an additional term $m$, which reduces the lower bound. Analyzing why this term does not appear in Sudholt's analysis, we note that the $\gamma_{i,j}$ chosen in \cite{Sudholt13} are underestimating the true probability to jump to elements of the path that are more than $k$ steps (on the path) away. When this is corrected, as confirmed to us by the author, Sudholt's proof would also only show our bound below. Consequently, there is currently no proof for~\eqref{eq:sud}. 

\begin{theorem}
Consider the \oea on a long $k$-path of length $m$ with mutation rate $p \leq 1/2$ starting at the all-$0$ bit string (the start of the path).\footnote{This simplifying assumption about the start point was also made in \cite{Sudholt13}.}

Let $T$ be the (random) time for the \oea to find the optimum. Then
$$
E[T] \geq m \; \frac{1-2p}{p(1-p)^{n}} \; \frac{1-2p}{1-p} \; \left(1- m \left( \frac{p}{1-p}\right)^{k-1}\right)^m.
$$	
\end{theorem}
\begin{proof}
We are setting up to apply Theorem~\ref{thm:fbp_visit_lower}. We partition the search space in the canonical way such that, for all $i \leq m$ with $i > 0$, $A_i$ contains the only $i$-th point of the path and nothing else, and $A_0$ contains all points not on the path. In order to simplify the analysis, we will first change the behavior of the algorithm such that it discards any offspring which differs from its parent by at least $k$ bits. This will allow us to apply Theorem~\ref{thm:fbp_visit_lower} quickly and cleanly, afterwards we will show that the progress of this modified algorithm is very close to the progress of the original algorithm.

In this modified process, we first consider the probability $p_i$ to leave a given level $i < m$. For this, the algorithm has to jump up exactly $j < k$ fitness levels, which is achieved by flipping a specific set of $j$ bits; the probability for this is
\begin{align*}
p_i = \sum_{j=1}^{k-1} p^j(1-p)^{n-j}
 & \leq (1-p)^{n} \sum_{j=1}^{\infty} \left(\frac{p}{1-p}\right)^j\\
 & = (1-p)^{n} \frac{p/(1-p)}{1-p/(1-p)}\\
 & = p(1-p)^{n} \frac{1}{1-2p}.
\end{align*}

Next we consider the probability $v_i$ to visit a level $i$. We want to apply Lemma~\ref{lem:visitprob}, so let some $x \in A_{< i}$ be given, on level $\ell(x)$. Let $d = i-\ell(x)$.
Note that $d$ is the Hamming distance between $x$ and the unique point in $A_i$. Thus, in case of $d \geq k$, we have $\Pr[x \rightarrow A_i] = 0$, so suppose $d < k$. Then we have
\begin{align*}
\Pr\bigg[x \rightarrow A_i \,\bigg|\, x \rightarrow \bigcup_{j=i}^m A_j\bigg]
 & =  \frac{\Pr[x \rightarrow A_i]}{\Pr[x \rightarrow \bigcup_{j=i}^m A_j]}\\
 & =  \frac{p^{d}(1-p)^{n-d}}{\sum_{j=d}^k p^{j}(1-p)^{n-j}}\\
 & =  \frac{1}{\sum_{j=d}^k p^{j-d}(1-p)^{d-j}}\\
 & =  \frac{1}{\sum_{j=0}^{k-d} p^{j}(1-p)^{-j}}\\
 & \geq  \frac{1}{\sum_{j=0}^{\infty} p^{j}(1-p)^{-j}}\\
 & =  1 - \frac{p}{1-p}\\
 & = \frac{1-2p}{1-p}.
\end{align*}
By Lemma~\ref{lem:visitprob}, we can use this last term as $v_i$ in Theorem~\ref{thm:fbp_visit_lower} (it also fulfills the second condition of Lemma~\ref{lem:visitprob}, since the process starts deterministically in the $0$ string). Note that neither $p_i$ nor $v_i$ depends on $i$. Using Theorem~\ref{thm:fbp_visit_lower} and recalling that we have $m$ levels, we get a lower bound of 
$$
m \frac{1-2p}{p(1-p)^{n}} \frac{1-2p}{1-p}.
$$
Note that this is exactly the term derived in~\cite{Sudholt13} except for a term correcting for the possibility of jumps of more than $k$ bits, which we also still need to correct for.

We now show that this probability of making a successful jump of distance at least $k$ is small. To that end we will show that it is very unlikely to leave a fitness level with a large jump rather than just move to the next level.

Suppose the algorithm is currently at $x \in A_i$. Leaving $x$ with a jump of at least $k$ to a specific element on the path is less likely the longer the jump is (since $p \leq 1/2$). Thus, we can upper bound the probability of jumping to an element of the path which is more than $k$ away as $p^k(1-p)^{n-k}$. Thus, conditional on leaving the fitness level, the probability of leaving it with a $\geq k$-jump is
\begin{align*}
\Pr[x \rightarrow A_{\geq i+k} \mid x \rightarrow A_{>i}]
 & = \frac{\Pr[x \rightarrow A_{\geq i+k}]}{\Pr[x \rightarrow A_{>i}]}\\
 & \leq  \frac{mp^k(1-p)^{n-k}}{p(1-p)^{n-1}}\\
 & =  m \left( \frac{p}{1-p}\right)^{k-1}.
\end{align*}
Thus, the probability of never making an accepted jump of at least $k$ is bounded from below by the probability to, independently once for each of the $m$ fitness levels, leave the fitness level with a $1$-step rather than a jump of at least $k$:
$$
\left(1- m \left( \frac{p}{1-p}\right)^{k-1}\right)^m.
$$
By pessimistically assuming that the process takes a time of $0$ in case it ever makes an accepted jump of at least $k$, we can lower-bound the expected time of the original process to reach the optimum as the product of the expected time of the modified process times the probability to never make progress of $k$ or more.
\end{proof}

\section{Conclusion}

In this work, we proposed a simple and natural way to prove lower bounds via fitness level arguments. The key to our approach is that the true run time can be expressed as the sum of the waiting times to leave a fitness level, weighted with the probability that this level is visited at all. When applying this idea, usually the most difficult part is estimating the probabilities to visit the levels, but as our examples \leadingones, \onemax, jump functions, and long paths show, this is not overly difficult and clearly easier than setting correctly the viscosity parameters of the previous fitness level method for lower bounds. For this reason, we are optimistic that our method will be an effective way to prove other lower bounds in the future, most easily, of course, for problems where upper bounds were proven via fitness level arguments as well.

Our method makes most sense for elitist evolutionary algorithms even though by regarding the best-so-far individual any evolutionary algorithm gives rise to a non-decreasing level process (at the price that the estimates for the level leaving probabilities become weaker). We are optimistic that our method can be extended to non-elitist algorithms, though. We note that the level visit probability $v_i$ for an elitist algorithm is equal to the expected number of separate visits to this level (simply because each level is visited exactly once or never). When defining the $v_i$ as the expected number of times the $i$-th level is visited, our upper and lower bounds of Theorems~\ref{thm:fbp_visit_lower} and~\ref{thm:fbp_visit_upper} remain valid (the proof would use Wald's equation). We did not detail this in our work since our main focus were the elitist examples regarded in~\cite{Sudholt13}, but we are optimistic that this direction could be interesting to prove lower bounds also for non-elitist algorithms.

\newcommand{\etalchar}[1]{$^{#1}$}


\begin{thebibliography}{WVHM18}

\bibitem[ABD20]{AntipovBD20ppsn}
Denis Antipov, Maxim Buzdalov, and Benjamin Doerr.
\newblock First steps towards a runtime analysis when starting with a good
  solution.
\newblock In {\em Parallel Problem Solving From Nature, PPSN 2020, Part~II},
  pages 560--573. Springer, 2020.

\bibitem[ABD21]{AntipovBD21gecco}
Denis Antipov, Maxim Buzdalov, and Benjamin Doerr.
\newblock Lazy parameter tuning and control: choosing all parameters randomly
  from a power-law distribution.
\newblock In {\em Genetic and Evolutionary Computation Conference, GECCO 2021}.
  {ACM}, 2021.

\bibitem[AD20]{AntipovD20ppsn}
Denis Antipov and Benjamin Doerr.
\newblock Runtime analysis of a heavy-tailed $(1+(\lambda, \lambda))$ genetic
  algorithm on jump functions.
\newblock In {\em Parallel Problem Solving From Nature, PPSN 2020, Part~II},
  pages 545--559. Springer, 2020.

\bibitem[ADK20]{AntipovDK20}
Denis Antipov, Benjamin Doerr, and Vitalii Karavaev.
\newblock The $(1 + (\lambda,\lambda))$ {GA} is even faster on multimodal
  problems.
\newblock In {\em Genetic and Evolutionary Computation Conference, GECCO 2020},
  pages 1259--1267. {ACM}, 2020.

\bibitem[BBD21]{BenbakiBD21}
Riade Benbaki, Ziyad Benomar, and Benjamin Doerr.
\newblock A rigorous runtime analysis of the 2-{MMAS}$_{\mathrm{ib}}$ on jump
  functions: ant colony optimizers can cope well with local optima.
\newblock In {\em Genetic and Evolutionary Computation Conference, GECCO 2021}.
  {ACM}, 2021.

\bibitem[BDDV20]{BuzdalovDDV20}
Maxim Buzdalov, Benjamin Doerr, Carola Doerr, and Dmitry Vinokurov.
\newblock Fixed-target runtime analysis.
\newblock In {\em Genetic and Evolutionary Computation Conference, GECCO 2020},
  pages 1295--1303. {ACM}, 2020.

\bibitem[BDN10]{BottcherDN10}
S\"untje B{\"o}ttcher, Benjamin Doerr, and Frank Neumann.
\newblock Optimal fixed and adaptive mutation rates for the {L}eading{O}nes
  problem.
\newblock In {\em Parallel Problem Solving from Nature, PPSN 2010}, pages
  1--10. Springer, 2010.

\bibitem[CDEL18]{CorusDEL18}
Dogan Corus, Duc{-}Cuong Dang, Anton~V. Eremeev, and Per~Kristian Lehre.
\newblock Level-based analysis of genetic algorithms and other search
  processes.
\newblock {\em {IEEE} Transactions on Evolutionary Computation}, 22:707--719,
  2018.

\bibitem[COY17]{CorusOY17}
Dogan Corus, Pietro~S. Oliveto, and Donya Yazdani.
\newblock On the runtime analysis of the {O}pt-{IA} artificial immune system.
\newblock In {\em Genetic and Evolutionary Computation Conference, {GECCO}
  2017}, pages 83--90. {ACM}, 2017.

\bibitem[COY18]{CorusOY18fast}
Dogan Corus, Pietro~S. Oliveto, and Donya Yazdani.
\newblock Fast artificial immune systems.
\newblock In {\em Parallel Problem Solving from Nature, {PPSN} 2018, Part
  {II}}, pages 67--78. Springer, 2018.

\bibitem[DD16]{DoerrD16}
Benjamin Doerr and Carola Doerr.
\newblock The impact of random initialization on the runtime of randomized
  search heuristics.
\newblock {\em Algorithmica}, 75:529--553, 2016.

\bibitem[DDK18]{DoerrDK18}
Benjamin Doerr, Carola Doerr, and Timo K{\"{o}}tzing.
\newblock Static and self-adjusting mutation strengths for multi-valued
  decision variables.
\newblock {\em Algorithmica}, 80:1732--1768, 2018.

\bibitem[DDY20]{DoerrDY20}
Benjamin Doerr, Carola Doerr, and Jing Yang.
\newblock Optimal parameter choices via precise black-box analysis.
\newblock {\em Theoretical Computer Science}, 801:1--34, 2020.

\bibitem[DFK{\etalchar{+}}16]{DangFKKLOSS16}
Duc{-}Cuong Dang, Tobias Friedrich, Timo K{\"{o}}tzing, Martin~S. Krejca,
  Per~Kristian Lehre, Pietro~S. Oliveto, Dirk Sudholt, and Andrew~M. Sutton.
\newblock Escaping local optima with diversity mechanisms and crossover.
\newblock In {\em Genetic and Evolutionary Computation Conference, GECCO 2016},
  pages 645--652. {ACM}, 2016.

\bibitem[DFK{\etalchar{+}}18]{DangFKKLOSS18}
Duc{-}Cuong Dang, Tobias Friedrich, Timo K{\"{o}}tzing, Martin~S. Krejca,
  Per~Kristian Lehre, Pietro~S. Oliveto, Dirk Sudholt, and Andrew~M. Sutton.
\newblock Escaping local optima using crossover with emergent diversity.
\newblock {\em {IEEE} Transactions on Evolutionary Computation}, 22:484--497,
  2018.

\bibitem[DFW10]{DoerrFW10}
Benjamin Doerr, Mahmoud Fouz, and Carsten Witt.
\newblock Quasirandom evolutionary algorithms.
\newblock In {\em Genetic and Evolutionary Computation Conference, GECCO 2010},
  pages 1457--1464. ACM, 2010.

\bibitem[DFW11]{DoerrFW11}
Benjamin Doerr, Mahmoud Fouz, and Carsten Witt.
\newblock Sharp bounds by probability-generating functions and variable drift.
\newblock In {\em Genetic and Evolutionary Computation Conference, GECCO 2011},
  pages 2083--2090. ACM, 2011.

\bibitem[DJW98]{DrosteJW98ecj}
Stefan Droste, Thomas Jansen, and Ingo Wegener.
\newblock A rigorous complexity analysis of the ${(1+1)}$ evolutionary
  algorithm for separable functions with boolean inputs.
\newblock {\em Evolutionary Computation}, 6:185--196, 1998.

\bibitem[DJW02]{DrosteJW02}
Stefan Droste, Thomas Jansen, and Ingo Wegener.
\newblock On the analysis of the (1+1) evolutionary algorithm.
\newblock {\em Theoretical Computer Science}, 276:51--81, 2002.

\bibitem[DJW12]{DoerrJW12algo}
Benjamin Doerr, Daniel Johannsen, and Carola Winzen.
\newblock Multiplicative drift analysis.
\newblock {\em Algorithmica}, 64:673--697, 2012.

\bibitem[DJWZ13]{DoerrJWZ13}
Benjamin Doerr, Thomas Jansen, Carsten Witt, and Christine Zarges.
\newblock A method to derive fixed budget results from expected optimisation
  times.
\newblock In {\em Genetic and Evolutionary Computation Conference, GECCO 2013},
  pages 1581--1588. ACM, 2013.

\bibitem[DK19]{DoerrK19}
Benjamin Doerr and Timo K{\"{o}}tzing.
\newblock Multiplicative up-drift.
\newblock In {\em Genetic and Evolutionary Computation Conference, GECCO 2019},
  pages 1470--1478. {ACM}, 2019.

\bibitem[DK21]{DoerrK21gecco}
Benjamin Doerr and Timo K\"otzing.
\newblock Lower bounds from fitness levels made easy.
\newblock In {\em Genetic and Evolutionary Computation Conference, GECCO 2021}.
  {ACM}, 2021.

\bibitem[DKLL20]{DoerrKLL20}
Benjamin Doerr, Timo K{\"{o}}tzing, J.~A.~Gregor Lagodzinski, and Johannes
  Lengler.
\newblock The impact of lexicographic parsimony pressure for {ORDER/MAJORITY}
  on the run time.
\newblock {\em Theoretical Computer Science}, 816:144--168, 2020.

\bibitem[DL16]{DangL16algo}
Duc{-}Cuong Dang and Per~Kristian Lehre.
\newblock Runtime analysis of non-elitist populations: from classical
  optimisation to partial information.
\newblock {\em Algorithmica}, 75:428--461, 2016.

\bibitem[DLMN17]{DoerrLMN17}
Benjamin Doerr, Huu~Phuoc Le, R\'egis Makhmara, and Ta~Duy Nguyen.
\newblock Fast genetic algorithms.
\newblock In {\em Genetic and Evolutionary Computation Conference, GECCO 2017},
  pages 777--784. {ACM}, 2017.

\bibitem[Doe19]{Doerr19tcs}
Benjamin Doerr.
\newblock Analyzing randomized search heuristics via stochastic domination.
\newblock {\em Theoretical Computer Science}, 773:115--137, 2019.

\bibitem[Doe20a]{Doerr20gecco}
Benjamin Doerr.
\newblock Does comma selection help to cope with local optima?
\newblock In {\em Genetic and Evolutionary Computation Conference, GECCO 2020},
  pages 1304--1313. {ACM}, 2020.

\bibitem[Doe20b]{Doerr20bookchapter}
Benjamin Doerr.
\newblock Probabilistic tools for the analysis of randomized optimization
  heuristics.
\newblock In Benjamin Doerr and Frank Neumann, editors, {\em Theory of
  Evolutionary Computation: Recent Developments in Discrete Optimization},
  pages 1--87. Springer, 2020.
\newblock Also available at \url{https://arxiv.org/abs/1801.06733}.

\bibitem[DZ21]{DoerrZ21aaai}
Benjamin Doerr and Weijie Zheng.
\newblock Theoretical analyses of multi-objective evolutionary algorithms on
  multi-modal objectives.
\newblock In {\em Conference on Artificial Intelligence, {AAAI} 2021}. {AAAI}
  Press, 2021.
\newblock To appear.

\bibitem[Fel68]{Feller68}
William Feller.
\newblock {\em An Introduction to Probability Theory and Its Applications},
  volume~I.
\newblock Wiley, third edition, 1968.

\bibitem[FK13]{FeldmannK13}
Matthias Feldmann and Timo K{\"{o}}tzing.
\newblock Optimizing expected path lengths with ant colony optimization using
  fitness proportional update.
\newblock In {\em Foundations of Genetic Algorithms, {FOGA} 2013}, pages
  65--74. {ACM}, 2013.

\bibitem[GKS99]{GarnierKS99}
Josselin Garnier, Leila Kallel, and Marc Schoenauer.
\newblock Rigorous hitting times for binary mutations.
\newblock {\em Evolutionary Computation}, 7:173--203, 1999.

\bibitem[GW17]{GiessenW17}
Christian Gie{\ss}en and Carsten Witt.
\newblock The interplay of population size and mutation probability in the ${(1
  + \lambda)}$ {EA} on {OneMax}.
\newblock {\em Algorithmica}, 78:587--609, 2017.

\bibitem[GW18]{GiessenW18}
Christian Gie{\ss}en and Carsten Witt.
\newblock Optimal mutation rates for the ${(1 + \lambda)}$ {EA} on {OneMax}
  through asymptotically tight drift analysis.
\newblock {\em Algorithmica}, 80:1710--1731, 2018.

\bibitem[HPR{\etalchar{+}}18]{HwangPRTC18}
Hsien{-}Kuei Hwang, Alois Panholzer, Nicolas Rolin, Tsung{-}Hsi Tsai, and
  Wei{-}Mei Chen.
\newblock Probabilistic analysis of the (1+1)-evolutionary algorithm.
\newblock {\em Evolutionary Computation}, 26:299--345, 2018.

\bibitem[HW19]{HwangW19}
Hsien{-}Kuei Hwang and Carsten Witt.
\newblock Sharp bounds on the runtime of the {(1+1)} {EA} via drift analysis
  and analytic combinatorial tools.
\newblock In {\em Foundations of Genetic Algorithms, {FOGA} 2019}, pages 1--12.
  {ACM}, 2019.

\bibitem[HY01]{HeY01}
Jun He and Xin Yao.
\newblock Drift analysis and average time complexity of evolutionary
  algorithms.
\newblock {\em Artificial Intelligence}, 127:51--81, 2001.

\bibitem[J{\"{a}}g07]{Jagerskupper07}
Jens J{\"{a}}gersk{\"{u}}pper.
\newblock Algorithmic analysis of a basic evolutionary algorithm for continuous
  optimization.
\newblock {\em Theoretical Computer Science}, 379:329--347, 2007.

\bibitem[Joh10]{Johannsen10}
Daniel Johannsen.
\newblock {\em Random Combinatorial Structures and Randomized Search
  Heuristics}.
\newblock PhD thesis, Universit\"at des Saarlandes, 2010.

\bibitem[JW02]{JansenW02}
Thomas Jansen and Ingo Wegener.
\newblock The analysis of evolutionary algorithms -- a proof that crossover
  really can help.
\newblock {\em Algorithmica}, 34:47--66, 2002.

\bibitem[JZ14]{JansenZ14}
Thomas Jansen and Christine Zarges.
\newblock Performance analysis of randomised search heuristics operating with a
  fixed budget.
\newblock {\em Theoretical Computer Science}, 545:39--58, 2014.

\bibitem[Leh10]{Lehre10}
Per~Kristian Lehre.
\newblock Negative drift in populations.
\newblock In {\em Parallel Problem Solving from Nature, PPSN 2010}, pages
  244--253. Springer, 2010.

\bibitem[Leh11]{Lehre11}
Per~Kristian Lehre.
\newblock Fitness-levels for non-elitist populations.
\newblock In {\em Genetic and Evolutionary Computation Conference, {GECCO}
  2011}, pages 2075--2082. {ACM}, 2011.

\bibitem[Len20]{Lengler20bookchapter}
Johannes Lengler.
\newblock Drift analysis.
\newblock In Benjamin Doerr and Frank Neumann, editors, {\em Theory of
  Evolutionary Computation: Recent Developments in Discrete Optimization},
  pages 89--131. Springer, 2020.
\newblock Also available at \url{https://arxiv.org/abs/1712.00964}.

\bibitem[LOW19]{LissovoiOW19}
Andrei Lissovoi, Pietro~S. Oliveto, and John~Alasdair Warwicker.
\newblock On the time complexity of algorithm selection hyper-heuristics for
  multimodal optimisation.
\newblock In {\em Conference on Artificial Intelligence, {AAAI} 2019}, pages
  2322--2329. {AAAI} Press, 2019.

\bibitem[LS14]{LassigS14}
J{\"{o}}rg L{\"{a}}ssig and Dirk Sudholt.
\newblock General upper bounds on the runtime of parallel evolutionary
  algorithms.
\newblock {\em Evolutionary Computation}, 22:405--437, 2014.

\bibitem[LW14]{LehreW14}
Per~Kristian Lehre and Carsten Witt.
\newblock Concentrated hitting times of randomized search heuristics with
  variable drift.
\newblock In {\em International Symposium on Algorithms and Computation, ISAAC
  2014}, pages 686--697. Springer, 2014.

\bibitem[MRC09]{MitavskiyRC09}
Boris Mitavskiy, Jonathan~E. Rowe, and Chris Cannings.
\newblock Theoretical analysis of local search strategies to optimize network
  communication subject to preserving the total number of links.
\newblock {\em International Journal on Intelligent Computing and Cybernetics},
  2:243--284, 2009.

\bibitem[RA19]{RoweA19}
Jonathan~E. Rowe and Aishwaryaprajna.
\newblock The benefits and limitations of voting mechanisms in evolutionary
  optimisation.
\newblock In {\em Foundations of Genetic Algorithms, {FOGA} 2019}, pages
  34--42. {ACM}, 2019.

\bibitem[Rud96]{Rudolph96}
G{\"{u}}nter Rudolph.
\newblock How mutation and selection solve long path problems in polynomial
  expected time.
\newblock {\em Evolutionary Computation}, 4:195--205, 1996.

\bibitem[Rud97]{Rudolph97}
G{\"u}nter Rudolph.
\newblock {\em Convergence Properties of Evolutionary Algorithms}.
\newblock Verlag Dr.~Kov{\v a}c, 1997.

\bibitem[RW20]{RajabiW20}
Amirhossein Rajabi and Carsten Witt.
\newblock Self-adjusting evolutionary algorithms for multimodal optimization.
\newblock In {\em Genetic and Evolutionary Computation Conference, GECCO 2020},
  pages 1314--1322. {ACM}, 2020.

\bibitem[RW21a]{RajabiW21gecco}
Amirhossein Rajabi and Carsten Witt.
\newblock Stagnation detection in highly multimodal fitness landscapes.
\newblock In {\em Genetic and Evolutionary Computation Conference, GECCO 2021}.
  {ACM}, 2021.

\bibitem[RW21b]{RajabiW21evocop}
Amirhossein Rajabi and Carsten Witt.
\newblock Stagnation detection with randomized local search.
\newblock In {\em Evolutionary Computation in Combinatorial Optimization,
  EvoCOP 2021}, pages 152--168. Springer, 2021.

\bibitem[Sud09]{Sudholt09}
Dirk Sudholt.
\newblock The impact of parametrization in memetic evolutionary algorithms.
\newblock {\em Theoretical Computer Science}, 410:2511--2528, 2009.

\bibitem[Sud10]{Sudholt10}
Dirk Sudholt.
\newblock General lower bounds for the running time of evolutionary algorithms.
\newblock In {\em Parallel Problem Solving from Nature, PPSN 2010, Part {I}},
  pages 124--133. Springer, 2010.

\bibitem[Sud13]{Sudholt13}
Dirk Sudholt.
\newblock A new method for lower bounds on the running time of evolutionary
  algorithms.
\newblock {\em {IEEE} Transactions on Evolutionary Computation}, 17:418--435,
  2013.

\bibitem[Weg01]{Wegener01}
Ingo Wegener.
\newblock Theoretical aspects of evolutionary algorithms.
\newblock In {\em Automata, Languages and Programming, {ICALP} 2001}, pages
  64--78. Springer, 2001.

\bibitem[Weg02]{Wegener02}
Ingo Wegener.
\newblock Methods for the analysis of evolutionary algorithms on
  pseudo-{B}oolean functions.
\newblock In Ruhul Sarker, Masoud Mohammadian, and Xin Yao, editors, {\em
  Evolutionary Optimization}, pages 349--369. Kluwer, 2002.

\bibitem[Wit13]{Witt13}
Carsten Witt.
\newblock Tight bounds on the optimization time of a randomized search
  heuristic on linear functions.
\newblock {\em Combinatorics, Probability {\&} Computing}, 22:294--318, 2013.

\bibitem[Wit14]{Witt14}
Carsten Witt.
\newblock Fitness levels with tail bounds for the analysis of randomized search
  heuristics.
\newblock {\em Information Processing Letters}, 114:38--41, 2014.

\bibitem[WVHM18]{WhitleyVHM18}
Darrell Whitley, Swetha Varadarajan, Rachel Hirsch, and Anirban Mukhopadhyay.
\newblock Exploration and exploitation without mutation: solving the jump
  function in ${\Theta(n)}$ time.
\newblock In {\em Parallel Problem Solving from Nature, {PPSN} 2018, Part
  {II}}, pages 55--66. Springer, 2018.

\end{thebibliography}

}

\end{document}